\documentclass[letterpaper, 10 pt, conference]{ieeeconf}  

\IEEEoverridecommandlockouts       

\usepackage{graphicx}   
\usepackage{epsfig} 
\usepackage[tight,footnotesize]{subfigure}
\usepackage[cmex10]{amsmath}   \interdisplaylinepenalty=2500
\usepackage{amssymb}
\usepackage{amsmath}

\DeclareMathOperator*{\minimize}{minimize}
\usepackage{amsopn}
\usepackage{multirow}
\usepackage{listings}
\usepackage{bigstrut}
\usepackage{xspace}
\usepackage{xparse}
\usepackage{float}
\usepackage{color}
\usepackage{hyperref}
\usepackage{verbatim}
\usepackage{booktabs}
\usepackage{algorithmic}
\usepackage{algorithm}
\usepackage[keeplastbox]{flushend}

\newcommand{\figref}[1]{Figure~\ref{#1}}
\newcommand{\tabref}[1]{Table~\ref{#1}}
\newcommand{\egref}[1]{Example~\ref{#1}}

\newtheorem{example}{Example}
\newtheorem{problem}{Problem}
\newtheorem{theorem}{Theorem}

\def\Rset{\mathbb{R}}

\def\act{{Act}}

\def\sinit{{\overline{s}}}

\def\Rset{\mathbb{R}}

\newcommand{\PONE}[1]{\textsf{Player~1}}
\newcommand{\PTWO}[1]{\textsf{Player~2}}

\newcommand{\ra}[1]{\renewcommand{\arraystretch}{#1}}

\newcommand{\startpara}[1]{{\vskip1pt\noindent{\bf #1.}}}

\makeatletter
\newcommand\footnoteref[1]{\protected@xdef\@thefnmark{\ref{#1}}\@footnotemark}
\makeatother

\begin{document}
%
\title{\LARGE \bf
Counterexamples for Robotic Planning \\
Explained in Structured Language}

%
%
%

\author{Lu Feng$^{1}$, 
	Mahsa Ghasemi$^{2}$, 
	Kai-Wei Chang$^{3}$, 
	and Ufuk Topcu$^{4}$
\thanks{$^{1}$Lu Feng is with the Department of Computer Science,
	University of Virginia, Charlottesville, VA 22904, USA
        {\tt\small lu.feng@virginia.edu}}%
\thanks{$^{2}$Mahsa Ghasemi is with the Department of Mechanical Engineering,
	The University of Texas at Austin, Austin, TX 78712-1221,  USA
        {\tt\small mahsa.ghasemi@utexas.edu}}%
\thanks{$^{3}$Kai-Wei Chang is with the Department of Computer Science, UCLA, Los Angeles, CA 90095-1596, USA
        {\tt\small kwchang.cs@ucla.edu}}%
\thanks{$^{4}$Ufuk Topcu is with the Department of Aerospace Engineering and Engineering Mechanics,
	The University of Texas at Austin, Austin, TX 78712-1221,  USA
        {\tt\small utopcu@utexas.edu}}%
\thanks{This work was supported in part by NASA grant \# NNX17AD04G, NSF grant \# 1651089, AFRL grant \# FA8650-15-C-2546, and ONR grant \# N00014-15-IP-00052.}%
}

\maketitle
\thispagestyle{empty}
\pagestyle{empty}

\begin{abstract}
Automated techniques such as model checking have been used to verify models of robotic mission plans based on Markov decision processes (MDPs) and generate counterexamples that may help diagnose requirement violations.
However, such artifacts may be too complex for humans to understand, because existing representations of counterexamples typically include a large number of paths or a complex automaton. 
To help improve the interpretability of counterexamples, we define a notion of \emph{explainable counterexample}, which includes a set of structured natural language sentences to describe the robotic behavior that lead to a requirement violation in an MDP model of robotic mission plan. 
We propose an approach based on mixed-integer linear programming for generating explainable counterexamples that are minimal, sound and complete. We demonstrate the usefulness of the proposed approach via a case study of warehouse robots planning.

\end{abstract}

\section{Introduction}\label{sec:intro}
Formal methods such as model checking~\cite{BK08} have recently been used to verify human-generated robotic mission plans against a set of requirements~\cite{Humphrey+2013}. 
In cases in which the plans may violate the requirements, such techniques generate \emph{counterexamples} that illustrate requirement violations and provide valuable diagnostic information~\cite{WJA+14,feng2016human}. 
Nevertheless, these artifacts may be too complex for humans to understand, because existing notions of counterexamples are defined as either a set of finite paths or an automaton typically with large number of states and transitions.
The objective of this paper is to generate \emph{explainable counterexamples} with structured language descriptions. Suppose that a robotic mission plan is captured by a Markov decision process (MDP). 
We formulate the main problem as finding an explainable counterexample that 
(1) is a counterexample in the MDP illustrating the requirement violation, 
and (2) can be described using a minimal set of structured language sentences out of predefined set of templates.

Typical requirements that can be verified on MDPs include safety properties, such as ``the maximum probability to reach an error state is at most 0.1''. 
A single path reaching an error state in the MDP may not suffice to illustrate the requirement violation (i.e., the probability to reach an error state exceeds the threshold 0.1). 
Instead, a more informative counterexample may contain \emph{a set of paths} all reaching the error state and carrying a total probability mass greater than the threshold.  
However, the number of such paths can be excessive in many cases (e.g., doubly exponential in the problem size~\cite{han2009counterexample}).
An alternative way to represent counterexamples is using \emph{critical subsystem} of an MDP, which is a subsystem of the original MDP such that the probability to reach an error state inside that exceeds the probability threshold~\cite{WJA+14}.
It is challenging to explain counterexamples represented either as a set of paths or a critical subsystem, even with a modest number of states.
In this paper, we propose a natural language medium for conveying the counterexamples so as to enhance their effectiveness for diagnostic purposes.

We formalize the notion of \emph{explanations} of counterexamples by defining the connection between explanations and MDP states and actions. 
Each counterexample explanation contains a set of structured language sentences, that are instantiated from predefined language templates using vocabularies relevant to the domain of interest.  
In this paper, we use warehouse robots as a running example, but the proposed approach also applies to other robotic applications in which similar vocabulary and templates can be established. 
We define the minimality of explanations in terms of the number of sentences.
We also define the soundness and completeness of counterexample explanations. 
Given an MDP encompassing a robotic mission plan, a desired requirement along with the probability threshold, and a set of domain-specific language templates, the main problem is how to compute minimal, sound and complete explainable counterexamples that illustrate the requirement violation in the MDP. 

The proposed solution is based on mixed-integer linear programming (MILP). 
The MILP objective is to minimize the number of sentences (i.e., the instantiations of structured language templates) in the explanation. The constraints encode the requirement violation (i.e., the reachability probability exceeds the threshold), the MDP transition relation and nondeterministic choices of actions, and the connection between sentences and MDP states and actions. 
The MILP results in a minimal set of sentences that should be included in the counterexample explanation, and a set of states to form a critical subsystem. 
Therefore, the approach seeks to search for a counterexample and its structured language explanation simultaneously. 
The set of sentences identified in the MILP are unordered. 
To make it easier to follow the explanation, we propose an algorithm to order these sentences based on the topological sorting of counterexample states. 
Finally, we demonstrate the usefulness of the approach on a case study of warehouse robots planning.

\startpara{Contributions}
We summarize the major contributions of this paper as follows:
\begin{enumerate}
\item A formalization of the notion of explainable counterexamples for MDPs, including definitions on the minimality, soundness and completeness of explanations.
\item An MILP approach to find explainable counterexamples with minimal, sound and complete explanations. 
\item A case study of warehouse robots planning to show the usefulness of the proposed approach.
\end{enumerate}

\startpara{Related Work}
Counterexample generation for model checking MDPs has been studied in several works using different representations of counterexamples:
\cite{han2009counterexample} computes the smallest number of paths in MDP whose joint probability mass exceeds the threshold and formulates the counterexample generation as a k-shortest path problem;
\cite{WJA+14} computes a critical subsystem of MDP with the minimal number of states and proposes solutions based on mixed-integer linear programming and SAT-modulo-theories. 
There are several attempts to generate human-readable counterexamples:
\cite{wimmer2013high} computes a minimal fragment of model description in some high-level modeling language (e.g., probabilistic guarded commands),
while \cite{feng2016human} computes structured probabilistic counterexamples as a sequence of ``plays'' that capture the high-level objectives in UAV mission planning. 
However, none of these approaches generates counterexamples that are interpretable by humans as easy as natural language explanations. 
This paper aims to fill this gap by presenting an approach to compute counterexamples for robotic planning explained in structured natural language. 

The study of natural language for robotic applications has mostly focused on translating human instructions expressed in natural language to robotic control commands. 
For example, \cite{hayes2017improving} considers the problem of synthesizing natural language descriptions of robotic control policy. In \cite{lignos2015provably}, the authors present an integrated system for synthesizing reactive controllers using natural language specifications which are translated into linear temporal logic formulas. If unsynthesizable, the minimal unsynthesizable core is returned as a subset of the natural language input specifications.
Nevertheless, the connection between counterexamples for MDP models and natural language explanations has not yet been explored.

\section{Problem Formulation} \label{sec:problem}
In this section, we provide the necessary background and the formal definition of the problem. 

\subsection{Counterexamples for MDPs}
Markov decision processes (MDPs) have been widely used as a modeling formalism in robot planning to represent abstract robotic mission plans \cite{Thrun05}. 
We denote an MDP as a tuple $\mathcal{M}=(S, \sinit, \act, P, L)$
where $S$ is a countable set of \emph{states},
$\sinit \in S$ is an \emph{initial state},
$\act$ is a finite set of \emph{actions},
$P: S \times \act \times S \to [0, 1] \subseteq \Rset$ is a \emph{transition relation}
such that $\forall s\in S, \; \forall \alpha \in \act: \sum_{s'\in S}P(s,\alpha,s')=1$,
and $L: S \to 2^{AP}$ is a \emph{labeling function} that assigns to each state the set of atomic propositions from $AP$ (the finite set of all atomic propositions) that hold true.
At each state $s$, first an action $\alpha \in \act$ is chosen nondeterministically, then a successor state $s'\in \mathsf{succ}(s,\alpha)$ is determined probabilistically based on $P(s,\alpha,s')$.
The nondeterministic choices of actions in an MDP are resolved by a \emph{strategy}, denoted by $\sigma: S \to \act$.
Model checking techniques~\cite{BK08} can be applied for automated verification of properties such as ``the probability of reaching error states is at most $\lambda$''. 
In case the property is violated, counterexamples can be generated as diagnostic feedback. 
Formally, a counterexample $\mathcal{M}^{c,\sigma}$ is a critical subsystem of the MDP $\mathcal{M}$ such that the probability to reach target states inside this subsystem under the strategy $\sigma$ exceeds the probability threshold $\lambda$ defined by a task.
This constraint satisfaction problem can be modeled as a mixed-integer linear program (MILP)~\cite{WJA+14}.

\begin{example}

\begin{figure}[t]
\centering
\includegraphics[width=.35\columnwidth]{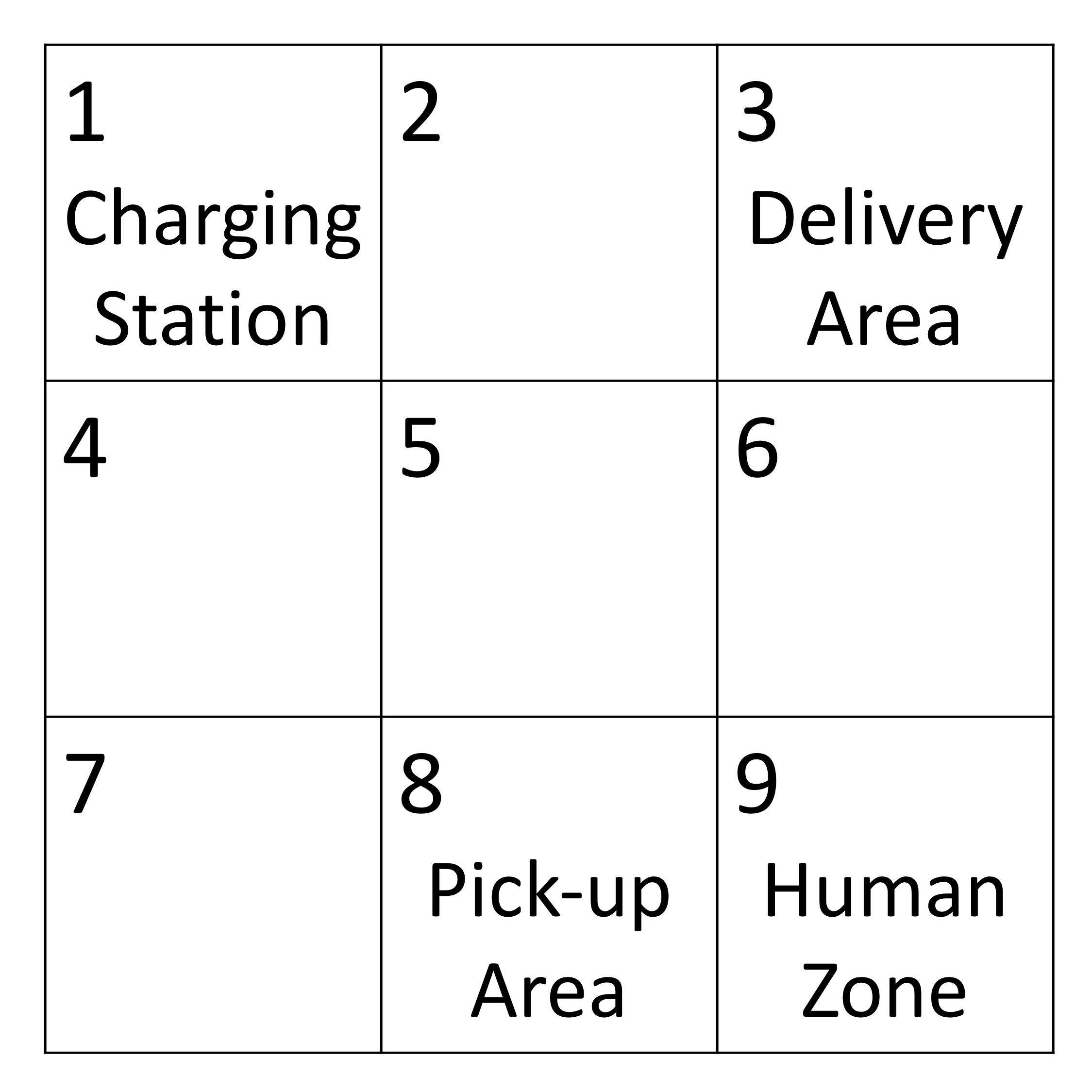}
\caption{A grid map for warehouse robots.}
\label{fig:grid}
\end{figure}

\begin{figure}[t]
\centering
\includegraphics[width=.95\columnwidth]{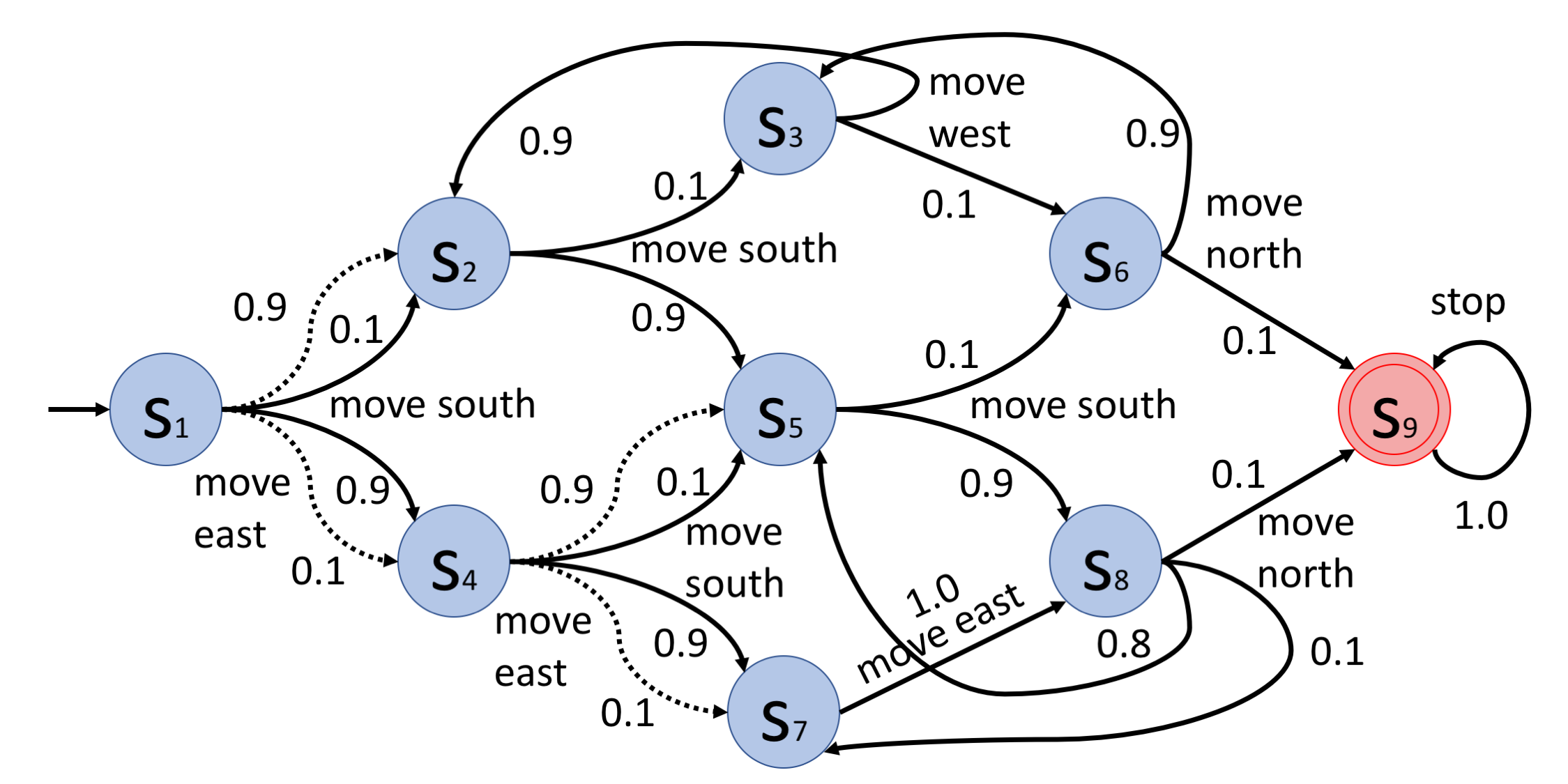}
\caption{An MDP representing a robotic mission plan based on \figref{fig:grid}.}
\label{fig:mdp}
\end{figure}

\begin{figure}[t]
\centering
\includegraphics[width=.95\columnwidth]{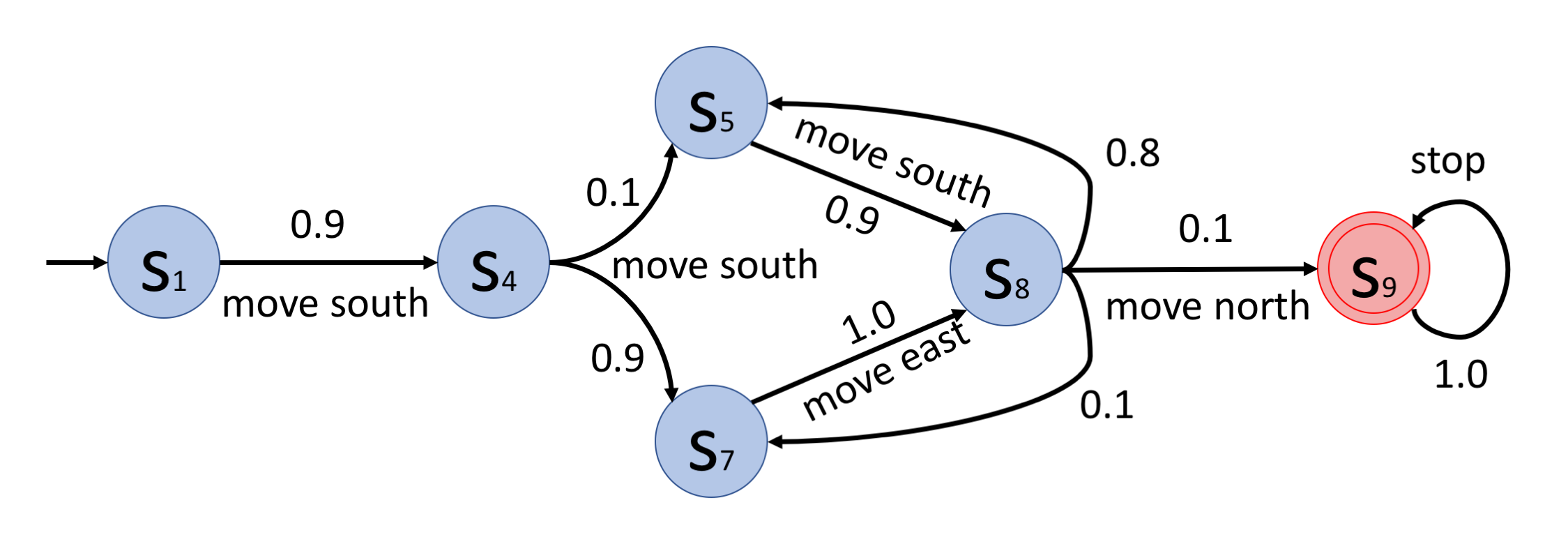}
\caption{A counterexample showing the violation of property ``the robot enters the human zone with probability at most 0.3''.}
\label{fig:cex}
\end{figure}

\figref{fig:grid} shows a grid map of a warehouse, in which the robot can take one of the five possible actions: ``move east'', ``move south'', ``move west'', ``move north'', and ``stop''.  
\figref{fig:mdp} shows an example MDP representing a robotic mission plan in this warehouse, where each state corresponds to a cell.  
Starting from the initial state $s_1$, the robot has a nondeterministic choice of ``move east'' (dashed lines) or ``move south'' (solid lines). 
Suppose the robot chooses to move east, then it arrives at the east neighbor $s_2$ with probability 0.9 and arrives at the south neighbor $s_4$ with probability 0.1 (e.g., due to perception uncertainties).
Each MDP state is labeled with one or more atomic propositions representing the robot's (relative) position. For example, $s_1$ is labeled with ``in charging station'' and ``north of pick-up area''.
The robot's mission objective is to first go to the pick-up area ($s_8$) to pick parts and then deliver parts to the delivery area ($s_3$). The robot stops if it enters the human zone ($s_9$). Consider the property ``the robot enters the human zone with probability at most 0.3''. \figref{fig:cex} shows a counterexample for this property, which is a critical subsystem with 6 states. 

Notice that, although the number of states in this subsystem is minimal in computation~\cite{WJA+14}, it is still hard to interpret.
In the following, we introduce structured language templates that improve the interpretability of such subsystems.

\label{eg:mdp}
\end{example}

\subsection{Structured Language Template}

A structured language template represents the core sentence structure over an ordered set of action and proposition pairs and can be used to facilitate the sentence generation process. 
We consider the following language templates for generating sentences to describe robotic behavior:
$$
\mbox{The robot }{\langle \mbox{action} \rangle}
\mbox{ when } {\langle \mbox{proposition} \rangle}.
$$
The above template can be instantiated by replacing ${\langle \mbox{action} \rangle}$ with possible robotic actions, and replacing ${\langle \mbox{proposition} \rangle}$ with atomic propositions or a conjunction of atomic propositions representing the robot's configuration. 
Therefore, a template can be viewed as a function that takes a set of actions and propositions as the input and outputs the corresponding sentence, denoted by $\mathcal{T}: (\alpha,\{\iota\}) \to Sen$ where $\alpha \in \act$ is an action, $\{\iota\} \subseteq AP$ is a set of atomic propositions, and $Sen$ is a sentence.
Given a tuple $(\alpha,\{\iota\})$, we can instantiate a corresponding sentence using the structured language template $\mathcal{T}$, which describes the robotic behavior of taking action $\alpha$ when under certain conditions captured by the set of atomic propositions $\{\iota\}$.

In this paper, we only consider one template as described above. Templates in other forms can be used through a similar procedure. In practice, the templates, actions and propositions are domain-specific. For concreteness, we will follow the warehouse example for the rest of the paper.  

\begin{example}
For the robotic mission in \egref{eg:mdp}, we can 
generate a set of sentences by instantiating the structured language template with the given MDP actions and propositions. The followings are a few example sentences that describe the robotic behavior captured in the MDP model.
\begin{itemize}
\item The robot \emph{moves east} when \emph{in charging station}. 
\item The robot \emph{moves south} when \emph{north of pick-up area}. 
\item The robot \emph{moves north} when \emph{south of delivery area} and \emph{north of human zone}.
\end{itemize}

\label{eg:nlp}
\end{example}

\subsection{Problem Statement}
We seek to compute succinct \emph{explainable counterexamples} for MDPs 
with \emph{explanations} in structured language to describe robotic behaviors that lead to requirement violation. 
In the following, we define the soundness, completeness and minimality of such counterexample explanations precisely. 
Given an MDP $\mathcal{M}=(S, \sinit, \act, P, L)$, a requirement $\phi$, and a structured language template $\mathcal{T}: (\alpha,\{\iota\}) \to Sen$, 
we say that a counterexample explanation $C \subseteq \{Sen\}$ is \emph{sound} if each sentence $c \in C$ corresponds to the robotic behavior in one or more states $s$ of a counterexample $\mathcal{M}^{c,\sigma}$ illustrating the violation of $\phi$ in $\mathcal{M}$;
that is, if $c=\mathcal{T}(\alpha, \{\iota\})$, then $\alpha \in \act(s)$ and $\{\iota\} \subseteq L(s)$.
A counterexample explanation $C$ is \emph{complete} if the robotic behavior at each state $s$ in a counterexample $\mathcal{M}^{c,\sigma}$ is described using exactly one sentence $c \in C$, which is an instantiation of $\mathcal{T}$ with the action $\sigma(s)$ and a set of atomic propositions $\{\iota\} \subseteq L(s)$.
We define the \emph{minimality} of counterexample explanations in terms of $|C|$, i.e.,  the number of sentences.

Note that a counterexample explanation describes an over-approximation of the robotic behavior captured in a critical subsystem. 
Each sentence $c=\mathcal{T}(\alpha, \{\iota\})$ corresponds to a robotic action $\alpha$ taken at a state (or multiple states) $s$ with $\alpha \in \act(s)$ and $\{\iota\} \subseteq L(s)$. 
However, it is not necessary that action $\alpha$ is taken at 
all MDP states labeled with propositions $\{\iota\}$. 
For example, suppose a sentence ``the robot moves south when north of pick-up area'' explains the robotic behavior in one state in the counterexample, but the robot may not take the same action ``move south'' at every MDP state labeled with the proposition ``north of pick-up area''.

\begin{table}[t]
\centering
\caption{An explanation of the counterexample in \figref{fig:cex}}
\begin{tabular}{|l|}
\hline
($\mathsf{S1}$) The robot moves south when in charging station. \\
($\mathsf{S2}$) The robot moves south when south of charging station. \\
($\mathsf{S3}$) The robot moves south when north of pick-up area. \\ 
($\mathsf{S4}$) The robot moves east when west of pick-up area. \\ 
($\mathsf{S5}$) The robot moves north when in pick-up area. \\ 
($\mathsf{S6}$) The robot stops when in human zone. \\ 
\hline
\end{tabular}
\label{tab:cex}
\end{table}

\begin{example}
\tabref{tab:cex} shows a sound and complete explanation for the counterexample shown in \figref{fig:cex}. It contains 6 sentences, each of which explains the robotic behavior in one state of the counterexample. For example, sentence ($\mathsf{S1}$) corresponds to state $s_1$, which is labeled with a proposition ``in charging station'' and takes the action ``move south''. Similarly, sentences ($\mathsf{S2}$)-($\mathsf{S6}$) correspond to the behavior at states $s_4$, $s_5$, $s_7$, $s_8$ and $s_9$, respectively.
However, this is not a minimal explanation. We will show how to automatically generate a minimal counterexample explanation using only 4 sentences in \egref{eg:minimal}.

\label{eg:cex}
\end{example}

We now state the problem statement formally.
\begin{problem}
Given an MDP $\mathcal{M}$, a reachability requirement $\phi$ with probability threshold $\lambda$ (violated in $\mathcal{M}$), and a structured language template $\mathcal{T}$, compute a minimal counterexample explanation that is sound and complete.
\label{problem1}
\end{problem}

\section{Solution Approach} \label{sec:approach}
We propose a solution based on mixed-integer linear programming (MILP) to compute minimal, sound and complete counterexample explanations. 
An advantage of the proposed approach is that it computes a counterexample and generates its structured language explanation simultaneously. 

\subsection{MILP Formulation}

We define a real-valued variable $p_s \in [0,1] \subseteq \Rset$ for each state $s \in S$ to track the probability of reaching the target states, denoted by set $T$, within a counterexample subsystem.
To encode the strategy $\sigma$ for resolving the nondeterminism in MDP, we define a binary variable $\theta_{s,\alpha}$ for each state $s \in S$ and action $\alpha \in \act(s)$ enabled in state $s$, 
such that $\theta_{s,\alpha}=1$ if action $\alpha$ is chosen at state $s$ by the strategy $\sigma$. 
Additionally, we use a binary variable $\mu_{\alpha,\{\iota\}}$ to denote if a corresponding sentence instantiated by the structured language template $\mathcal{T}$ with the tuple $(\alpha, \{\iota\})$ is included in the counterexample explanation. 
Therefore, the number of real-valued variables used in the MILP formulation is given by the number of MDP states, and the number of binary variables is the summation of the number of MDP nondeterministic choices and the number of possible structured language sentences.

The resulting MILP problem is:
\begin{subequations}
\begin{align}
\minimize_{p_s\in [0,1], \theta_{s,\alpha}\in \{0,1\}, \mu_{\alpha,\{\iota\}}\in \{0,1\} } & \ \sum_{\mathcal{T}} \mu_{\alpha,\{\iota\}} 
\label{eq:obj} \\
\intertext{\hspace{70pt} subject to}
 & \ \ p_{\sinit} > \lambda,  
\label{eq:c1} \\
\forall s \in T: & \ \ p_s = 1, 
\label{eq:c2} \\
\forall s \in S \setminus T,  \ \forall \alpha \in \act(s): & \ \ \nonumber
p_s \le   (1 - \theta_{s,\alpha}) \ + \\ &\hspace{-4mm} \sum_{s' \in \mathsf{succ}(s,\alpha)} P(s,\alpha,s') \cdot p_{s'},   
\label{eq:c3} \\
\forall s \in S \setminus T: & \ \ p_s \le \sum_{\alpha \in \act(s)} \theta_{s,\alpha}, 
\label{eq:c4} \\
\forall s \in S \setminus T,  \forall \alpha \in \act(s): & \ \ 
\theta_{s,\alpha} \le \sum_{\{\iota\} \subseteq L(s)}{\mu_{\alpha, \{\iota\}}}. 
\label{eq:c5} 
\end{align}
\end{subequations}

\noindent The objective function (\ref{eq:obj}) is to minimize the number of sentences in the counterexample explanation. Each sentence is an instantiation of the structured language template $\mathcal{T}$ with the tuple $(\alpha, \{\iota\})$.
The constraint (\ref{eq:c1}) guarantees that the probability of reaching target states from the initial state $\sinit$ exceeds the threshold $\lambda$ of the property $\phi$, therefore a critical subsystem that violates the reachability probability in MDP $\mathcal{M}$ is found. 
The constraint (\ref{eq:c2}) ensures that the probability $p_s$ of a target state $s \in T$ is 1.
The constraint (\ref{eq:c3}) encodes the probabilistic transition relation of the MDP. 
The constraint (\ref{eq:c4}) ensures that one action is chosen in each state by the strategy.
Finally, the constraint (\ref{eq:c5}) requires that if an action is chosen in a state, then at least one sentence has to be instantiated to explain this behavior.

\begin{theorem}
The solution to the MILP (\ref{eq:obj})-(\ref{eq:c5}) solves Problem~\ref{problem1}.

\end{theorem}
(The proof of the theorem is given in the appendix.)
\begin{example}
Consider the MDP shown in \figref{fig:mdp} and the property ``the robot enters the human zone with probability at most 0.3''. 
We write the MILP formulation to compute an explainable counterexample with the minimal, sound and complete explanation.
We introduce 9 real-valued variables $p_{s_1},\dots,p_{s_9}$ to represent the probability of reaching the target state $s_9$ from each state in the MDP. We use 11 binary variable $\theta_{s,\alpha}$ to encode the nondeterministic choices of actions in each state. For example, $\theta_{s_1,\alpha_1}$ and $\theta_{s_1,\alpha_2}$ represent the choices of ``move east'' and ``move south'' in state $s_1$, respectively. 
Suppose that we only consider sentences instantiated from the structured language template $\mathcal{T}$ with a single atomic proposition. We use 60 binary variables $\mu_{\alpha, \iota}$ to represent all possible template instances given by the 5 actions and 12 propositions. 
The MILP minimizes
$$
\sum_{\alpha \in \act, \iota \in AP} \mu_{\alpha,\{\iota\}}.
$$
Since the initial state is $s_1$ and the probability threshold for the property is 0.3, we write the constraint (\ref{eq:c1}) as
$
p_{s_1} > 0.3.
$
The target state is $s_9$, thus the constraint (\ref{eq:c2}) is
$
p_{s_9} = 1.
$
The constraint (\ref{eq:c3}) encodes the MDP transition relations, for instance, the encoding for $s_1$ is as follows:  
\begin{equation*}
\begin{aligned}
p_{s_1} &\le (1-\theta_{s_1,\alpha_1}) + 0.9 p_{s_2} + 0.1 p_{s_4}, \mbox{ and} \\
p_{s_1} &\le (1-\theta_{s_1,\alpha_2}) + 0.1 p_{s_2} + 0.9 p_{s_4}. 
\end{aligned}
\end{equation*}
The constraint (\ref{eq:c4}) for encoding the nondeterministic choices in $s_1$ is
$$
p_{s_1} \le \theta_{s_1,\alpha_1} + \theta_{s_1,\alpha_2}.
$$
The constraints for other states can be written similarly.
We only give an example encoding of the constraint (\ref{eq:c5}) as follows.
Suppose state $s_1$ is labeled with two atomic propositions: $\iota_1$ - ``in charging station'' and $\iota_2$ - ``north of pick-up area''.
We encode the constraint (\ref{eq:c5}) for state $s_1$ as
\begin{equation*}
\begin{aligned}
\theta_{s_1,\alpha_1} &\le \mu_{\alpha_1, \{\iota_1\}} + \mu_{\alpha_1, \{\iota_2\}}, \mbox{ and} \\
\theta_{s_1,\alpha_2} &\le \mu_{\alpha_2, \{\iota_1\}} + \mu_{\alpha_2, \{\iota_2\}}.
\end{aligned}
\end{equation*}

Solving the resulting MILP for this example yields a counterexample as shown in \figref{fig:cex}, and a minimal explanation with 4 sentences, 
which are ($\mathsf{S3}$)-($\mathsf{S6}$) shown in \tabref{tab:cex}.
Sentence ($\mathsf{S3}$) describes the robot's behavior at states $s_1$, $s_4$ and $s_5$, because they take the same action ``move south'' and have a common proposition ``north of pick-up area''.
Sentences ($\mathsf{S4}$), ($\mathsf{S5}$), and ($\mathsf{S6}$) correspond to the robot's behavior at states $s_7$, $s_8$, and $s_9$, respectively.
Notice that, compared with \egref{eg:cex}, the generated explanation uses fewer sentences. Sentences ($\mathsf{S1}$) and ($\mathsf{S2}$) are used in \egref{eg:cex} to represent the robotic behavior in states $s_1$ and $s_4$, but they are omitted in the explanation resulting from MILP since ($\mathsf{S3}$) also captures the behavior of those two states.

\label{eg:minimal}
\end{example}

\subsection{Sentence Instantiation and Ordering}

\begin{algorithm}[t]
\caption{Instantiate and order explanation sentences}
\label{alg:order}
\begin{algorithmic}
    \REQUIRE An MDP $\mathcal{M}$, a structured language template $\mathcal{T}$, and MILP results $\{p_{s}>0, \ \mu_{\alpha,\{\iota\}}=1, \ \theta_{s,\alpha}=1\}$.
    \ENSURE An ordered list of sentences $C$.
	\STATE $C$ $\leftarrow$ Empty list;
    \STATE $N$ $\leftarrow$ $\{\sinit\}$;
    \WHILE{$N$ is non-empty}
    \STATE remove a state $s$ from the head of $N$;
    \STATE find the action $\alpha \in \act(s)$ such that $\theta_{s,\alpha}=1$;
    \STATE given $\alpha$, find $\mu_{\alpha,\{\iota\}}=1$ such that $\{\iota\} \subseteq L(s)$;
    \STATE instantiate a sentence $Sen \ \leftarrow \ \mathcal{T}(\alpha, \{\iota\})$;
    \IF{$C$ does not already contain $Sen$}
    \STATE add $Sen$ to the tail of $C$
    \ENDIF 
    \FORALL{state $s'$ with $p_{s'}>0$ and $P(s,\alpha,s')>0$} 
    \STATE insert $s'$ to the tail of $N$
    \ENDFOR
    \ENDWHILE
    \RETURN $C$
\end{algorithmic}
\end{algorithm}

The results of the MILP yield a set of tuples $(\alpha,\{\iota\})$ with the binary variables $\mu_{\alpha,\{\iota\}}=1$.
Each tuple corresponds to a sentence $\mathcal{T}(\alpha, \{\iota\})$ instantiated using the structured language template $\mathcal{T}$. 
However, the MILP formulation does not determine the ordering of these sentences. 
To help humans better understand the counterexample explanations, we need to present these sentences in a meaningful order. 
Algorithm~\ref{alg:order} presents a procedure of instantiating and ordering explanation sentences, based on the topological sorting of MDP states. 
We demonstrate the usage of this algorithm via the following example.
 
\begin{example}
The MILP results of \egref{eg:minimal} yield 6 counterexample states ($s_1$, $s_4$, $s_5$, $s_7$, $s_8$, $s_9$) and 4 sentences ($\mathsf{S3}$)-($\mathsf{S6}$). 
Starting from the initial state $s_1$, we identify sentence ($\mathsf{S3}$) from the MILP results to describe the robotic behavior at $s_1$. Therefore, we add ($\mathsf{S3}$) to the ordered list. We remove $s_1$ from the set $N$ and add its successor state $s_4$ to $N$. Next, we consider state $s_4$ and find that only sentence ($\mathsf{S3}$) describes its behavior. Since ($\mathsf{S3}$) is already in the list, we continue to remove $s_4$ from the set $N$ and add its successor states $s_5$ and $s_7$ to $N$.
The algorithm continues until it processes all states in the critical subsystem. After running this procedure, the ordering of sentences in this counterexample explanation is obtained as 
$(\mathsf{S3}) \rightarrow (\mathsf{S4}) \rightarrow (\mathsf{S5}) \rightarrow (\mathsf{S6})$.

\label{eg:sorting}
\end{example}

\section{Experimental Evaluation} \label{sec:exp}
\begin{table*}[!htbp]\centering\scriptsize

\caption{Experimental results for the warehouse robot planning}
\ra{1}
\begin{tabular*}{1\linewidth}{cccccccccccccccc@{}}
\toprule
& \phantom{00} & \multicolumn{2}{c}{MDP Size} & \phantom{00} & \multicolumn{4}{c}{Approach in~\cite{WJA+14}} & \phantom{00} & \multicolumn{5}{c}{Proposed Approach} \\
\cmidrule{3-4} \cmidrule{6-9} \cmidrule{11-15} 
$N$
&& \# States & \# Transitions 
&& \begin{tabular}{@{}c@{}}\# Binary\\ Variables\end{tabular} & \begin{tabular}{@{}c@{}}\# Real\\ Variables\end{tabular}  & \# States & \begin{tabular}{@{}c@{}} Time (s) \end{tabular} 
&& \begin{tabular}{@{}c@{}}\# Binary\\ Variables\end{tabular} & \begin{tabular}{@{}c@{}}\# Real\\ Variables\end{tabular}  & \# States & \# Sentences & \begin{tabular}{@{}c@{}} Time (s) \end{tabular} \\ 
\midrule
10 && 100 & 208 && 309 & 56 & 9 & 0.11 && 242 & 56 & 9 & 3 & 1.39  \\
20 && 400 & 788 && 1,369 & 156 & 19 & \textbf{24.86} && 1,022 & 156 & 39 & 3 & \textbf{4.43}  \\
30 && 900 & 1,768 && 2,652 & 900 & -- & \textbf{time-out} && 2,273 & 291 & 29 & 3 & \textbf{1.54}   \\
40 && 1,600 & 3,148 && 4,732  & 1,600  & -- & \textbf{time-out} && 4,231 & 506 & 79 & 3 & \textbf{17.45}  \\
50 && 2,500 & 4,928 && 7,412 & 2,500 & -- & \textbf{time-out} && 6,669 & 756 & 99 & 3 & \textbf{32.33}  \\

\bottomrule
\end{tabular*}
\label{tab:results}
\end{table*}

In this section, we first introduce a case study of warehouse robot planning, then discuss the results of applying the proposed approach to the case study.


\begin{figure}[!t]
\centering
\includegraphics[width=.9\columnwidth]{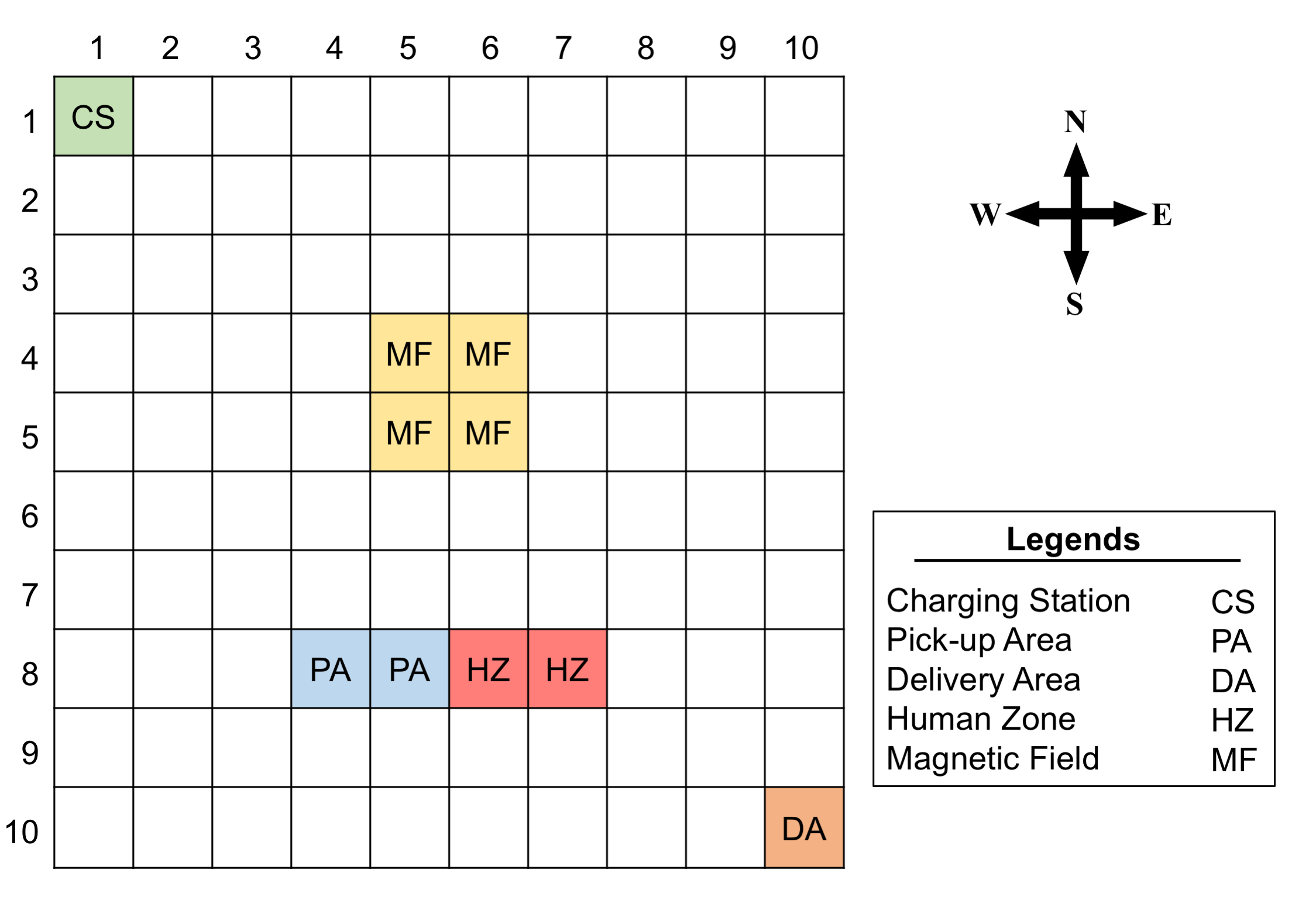}
\caption{A warehouse map with a $10\times10$ grid map.}
\label{fig:grid-world}
\end{figure}

The case study considers a warehouse map with an $N \times N$ grid map.  
\figref{fig:grid-world} shows an example map when $N=10$.
The map is annotated with a set of locations: charging station, pick-up area, delivery area, human zone, and magnetic field. 
The map scales up as $N$ increases, but the layout of locations does not change.  
For each map, we construct an MDP model to represent the robotic mission plan of picking parts from the pick-up area and transferring to the delivery area. Each state in the MDP corresponds to a cell in the map and is labeled with a set of atomic propositions, representing the robot's location. Example of such atomic propositions include ``in charging station'' and ``north of delivery area''. The robot may take one of the five possible actions: ``move east'', ``move south'', ``move west'', ``move north'', and ``stop''. At each MDP state, the robot makes a nondeterministic choice of action. These actions are executed correctly at most of the time. 
However, the accuracy of robot's sensors is affected in the magnetic field and it introduces uncertainty in robot's motion. Thus, with some known probability, the robot would execute the action incorrectly. For example, suppose the robot decides to move south, it would reach the south neighbor with probability 0.9, and reach the east neighbor with probability 0.1 due to inaccurate sensor readings in the area with magnetic field.  


We implemented the proposed approach using the PRISM model checker~\cite{KNP11} and Gurobi~\cite{gurobi} optimization toolbox.
For comparison, we also implemented the approach in~\cite{WJA+14} that computes a counterexample with the minimal number of states. 
\tabref{tab:results} summarizes the experimental results of the case study with five different maps ($N$ ranges from 10 to 50, by increment of 10).
Consider the safety property ``the robot enters the human zone with probability at most 0.1''.
For each scenario, we report the number of states and transitions of the MDP model, the number of states in the computed counterexample, and the number of sentences in the generated counterexample explanation. We also report the number of binary and real variables, and the running time for computing the MILP solution. 
The experiments were run on a laptop with 2.0 GHz Intel Core i7-4510U CPU and with 8.00 GB RAM. We impose a time-out of 1 hour. 
All MDP models and MILP formulations are available online.\footnote{\url{https://goo.gl/RVXXVN}}


In all cases, the proposed approach successfully computes a counterexample together with a succinct explanation in structured language. As the map scales up, the number of states in a counterexample also increases. It is interesting to note that the number of sentences in the generated explanations are the same, although, they use different sentences. 

By contrast, previous approaches (e.g., the approach in~\cite{WJA+14}) do not allow the automated generation of counterexample explanations. 
It is nontrivial to interpret counterexamples represented as subsystems.
In addition, the proposed approach has a better performance compared to the approach in~\cite{WJA+14} in the sense that it uses fewer number of binary and real variables in the MILP solution. Consequently, the running time to solve the optimization problem becomes much faster. 
\tabref{tab:results} also shows a growing size of MDP models as the map size ($N$) increases. 
As a result, the number of variables in MILP encoding increases in both methods. Nevertheless, for all five cases, counterexamples and explanations are computed by the proposed approach in less than 1 minute.
Therefore, we expect that the proposed approach would scale well for larger models as well.

\section{Conclusion} \label{sec:conclu}
In this paper, we proposed an approach to compute explainable counterexamples for robotic planning problems that are modeled as MDPs. The generated counterexample explanations are expressed as a minimal set of structured language sentences, which have the potential advantage of providing diagnostic feedback to humans. We also performed the experimental evaluation on several different warehouse robotic plans, with the MDP model size ranging from hundred to thousands of states. 
The results are very encouraging: in each of these cases, a succinct set of structured language sentences are automatically computed to show robotic behaviors that lead to the requirement violation. 
Given the promising computational properties of the proposed approach, in the future, we will investigate extensions to more expressive requirements and more flexible language templates that are closer to natural language.

\appendix
\setcounter{theorem}{0}
 \begin{theorem}\label{milp}
The solution to the MILP (\ref{eq:obj})-(\ref{eq:c5}) solves Problem~\ref{problem1}.
 \end{theorem}
\begin{proof}
(Sketch of proof.)
We need to prove that the solution to the MILP (\ref{eq:obj})-(\ref{eq:c5}) yields a minimal counterexample explanation that is sound and complete.

\startpara{Minimality} The minimality of the counterexample explanation is guaranteed by the objective function (\ref{eq:obj}), which minimizes the number of sentences in the explanation. 

\startpara{Soundness} By definition, a counterexample explanation is sound if each sentence in the explanation corresponds to the robotic behavior in one or more states of a counterexample violating the requirement.
We prove this by contradiction. 
Suppose there is a sentence in the MILP solution with $\mu_{\alpha, \{\iota\}}=1$, and there is only one MDP state $s$ satisfying $\alpha \in \act(s)$ and $\{\iota\} \subseteq L(s)$.
However, $s$ is not included in the counterexample, that is $p_s=0$ in the MILP solution. 
Suppose $\act(s)$ and $L(s)$ are both singular sets,
then we can write the MILP constraints (\ref{eq:c4}) and (\ref{eq:c5}) as follows:
$0 \le \theta_{s,\alpha}$ and $\theta_{s,\alpha} \le \mu_{\alpha, \{\iota\}}$, which reduces to 
$\mu_{\alpha, \{\iota\}} \ge 0$. 
The minimizing objective function (\ref{eq:obj}) enforces $\mu_{\alpha, \{\iota\}} = 0$, which is a contradiction with the assumption that $\mu_{\alpha, \{\iota\}} = 1$. 
Thus, the resulting counterexample explanation is sound.

\startpara{Completeness} By definition, a counterexample explanation is complete if the action taken at each counterexample state is described using exactly one sentence in the explanation.
Suppose $s$ is a counterexample state in the MILP solution, then $p(s)>0$. 
Hence, based on the MILP constraint (\ref{eq:c4}), at least one nondeterministic choice $\theta_{s,\alpha}$ should be true. 
Subsequently, the MILP constraint (\ref{eq:c5}) requires at least one sentence variable $\mu_{\alpha, \{\iota\}}$ to be true. 
Since the objective function (\ref{eq:obj}) is to minimize the number of sentences, the optimization solution will only choose one $\mu_{\alpha, \{\iota\}}$ to be true in order to describe state $s$.
Thus, the resulting counterexample explanation is complete.
\end{proof}




\bibliographystyle{IEEEtran}
\bibliography{explainCEX}

\begin{thebibliography}{10}
\providecommand{\url}[1]{#1}
\csname url@rmstyle\endcsname
\providecommand{\newblock}{\relax}
\providecommand{\bibinfo}[2]{#2}
\providecommand\BIBentrySTDinterwordspacing{\spaceskip=0pt\relax}
\providecommand\BIBentryALTinterwordstretchfactor{4}
\providecommand\BIBentryALTinterwordspacing{\spaceskip=\fontdimen2\font plus
\BIBentryALTinterwordstretchfactor\fontdimen3\font minus
  \fontdimen4\font\relax}
\providecommand\BIBforeignlanguage[2]{{%
\expandafter\ifx\csname l@#1\endcsname\relax
\typeout{** WARNING: IEEEtran.bst: No hyphenation pattern has been}%
\typeout{** loaded for the language `#1'. Using the pattern for}%
\typeout{** the default language instead.}%
\else
\language=\csname l@#1\endcsname
\fi
#2}}

\bibitem{BK08}
C.~Baier and J.-P. Katoen, \emph{Principles of Model Checking}.\hskip 1em plus
  0.5em minus 0.4em\relax MIT Press, 2008.

\bibitem{Humphrey+2013}
L.~Humphrey and M.~Patzek, ``Model checking human-automation {UAV} mission
  plans,'' in \emph{Proc. AIAA Guidance, Navigation, and Control Conf.}, 2013.

\bibitem{WJA+14}
R.~Wimmer, N.~Jansen, E.~{\'{A}}brah{\'{a}}m, J.~Katoen, and B.~Becker,
  ``Minimal counterexamples for linear-time probabilistic verification,''
  \emph{Theoretical Computer Science}, vol. 549, pp. 61--100, 2014.

\bibitem{feng2016human}
L.~Feng, L.~Humphrey, I.~Lee, and U.~Topcu, ``Human-interpretable diagnostic
  information for robotic planning systems,'' in \emph{Intelligent Robots and
  Systems (IROS), 2016 IEEE/RSJ International Conference on}.\hskip 1em plus
  0.5em minus 0.4em\relax IEEE, 2016, pp. 1673--1680.

\bibitem{han2009counterexample}
T.~Han, J.-P. Katoen, and D.~Berteun, ``Counterexample generation in
  probabilistic model checking,'' \emph{IEEE Transactions on Software
  Engineering}, vol.~35, no.~2, pp. 241--257, 2009.

\bibitem{wimmer2013high}
R.~Wimmer, N.~Jansen, A.~Vorpahl, E.~{\'A}brah{\'a}m, J.-P. Katoen, and
  B.~Becker, ``High-level counterexamples for probabilistic automata,'' in
  \emph{International Conference on Quantitative Evaluation of Systems}.\hskip
  1em plus 0.5em minus 0.4em\relax Springer, 2013, pp. 39--54.

\bibitem{hayes2017improving}
B.~Hayes and J.~A. Shah, ``Improving robot controller transparency through
  autonomous policy explanation,'' in \emph{Proceedings of the 2017 ACM/IEEE
  International Conference on Human-Robot Interaction}.\hskip 1em plus 0.5em
  minus 0.4em\relax ACM, 2017, pp. 303--312.

\bibitem{lignos2015provably}
C.~Lignos, V.~Raman, C.~Finucane, M.~Marcus, and H.~Kress-Gazit, ``Provably
  correct reactive control from natural language,'' \emph{Autonomous Robots},
  vol.~38, no.~1, pp. 89--105, 2015.

\bibitem{Thrun05}
S.~Thrun, W.~Burgard, and D.~Fox, \emph{Probabilistic Robotics (Intelligent
  Robotics and Autonomous Agents)}.\hskip 1em plus 0.5em minus 0.4em\relax The
  MIT Press, 2005.

\bibitem{KNP11}
M.~Kwiatkowska, G.~Norman, and D.~Parker, ``{PRISM} 4.0: Verification of
  probabilistic real-time systems,'' in \emph{Proc. 23rd International
  Conference on Computer Aided Verification (CAV'11)}, ser. LNCS,
  G.~Gopalakrishnan and S.~Qadeer, Eds., vol. 6806.\hskip 1em plus 0.5em minus
  0.4em\relax Springer, 2011, pp. 585--591.

\bibitem{gurobi}
``{Gurobi Optimization},'' \url{http://www.gurobi.com}.

\end{thebibliography}


\end{document}